\documentclass{article}

\usepackage{amsmath}
\usepackage{amsfonts}
\usepackage{amsthm}
\usepackage{amssymb}
\usepackage{booktabs}
\usepackage{multirow}
\usepackage{hyperref}
\usepackage{bm}
\usepackage{graphicx}
\usepackage[justification=centering]{caption}
\usepackage{biblatex}
\newcommand{\ind}[3]{({#1}\!\perp\kern-6pt\perp\!{#2}|{#3})}

\newcommand{\R}{\mathbb{R}}

\newcommand{\BB}{\mathbf{B}}
\newcommand{\CC}{\mathbf{C}}

\newcommand{\II}{\mathbf{I}}
\newcommand{\JJ}{\mathbf{J}}

\newcommand{\TT}{\mathbf{T}}
\newcommand{\UU}{\mathbf{U}}

\newcommand{\bb}{\mathbf{b}}
\newcommand{\cc}{\mathbf{c}}
\newcommand{\ee}{\mathbf{e}}
\newcommand{\bt}{\mathbf{t}}
\newcommand{\xx}{\mathbf{x}}

\newcommand{\zero}{\mathbf{0}}

\newcommand{\SSigma}{\bm{\Sigma}}

\newcommand{\cD}{\mathcal{D}}
\newcommand{\cE}{\mathcal{E}}
\newcommand{\cI}{\mathcal{I}}
\newcommand{\cJ}{\mathcal{J}}
\newcommand{\cM}{\mathcal{M}}
\newcommand{\cN}{\mathcal{N}}
\newcommand{\cP}{\mathcal{P}}
\newcommand{\cU}{\mathcal{U}}

\newcommand{\bp}{\varepsilon^\ast}
\newcommand{\ve}{\varepsilon}
\newcommand{\emt}{(\cJ_k, \cU_k)}
\newcommand{\Cxk}{\CC_\xx^k}
\newcommand{\Cxkh}{\hat\CC_\xx^k}
\newcommand{\See}{\SSigma_\ee}
\newcommand{\Seeh}{\hat\SSigma_\ee}
\newcommand{\BBh}{\hat\BB}
\newcommand{\llce}{(\hat\BB, \hat\SSigma_\ee)}

\DeclareMathOperator{\cov}{cov}

\DeclareMathOperator{\range}{range}

\DeclareMathOperator*{\argmin}{arg\,min}

\newtheorem{lemma}{Lemma}
\newtheorem{proposition}{Proposition}

\title{Robust Causal Analysis of Linear Cyclic Systems With Hidden Confounders}
\author{Boris Lorbeer\\Technical University Berlin\\10587 Berlin\\Germany \and Axel K\"upper\\Technical University Berlin\\10587 Berlin\\Germany}
\date{}

\begin{document}

\maketitle

\begin{abstract}
    We live in a world full of complex systems which we need to improve our
    understanding of. To accomplish this, purely probabilistic investigations
    are often not enough. They are only the first step and must be followed by
    learning the system's underlying mechanisms. This is what the discipline of
    causality is concerned with. Many of those complex systems contain feedback
    loops which means that our methods have to allow for cyclic causal
    relations. Furthermore, systems are rarely sufficiently isolated, which
    means that there are usually hidden confounders, i.e., unmeasured variables
    that each causally affects more than one measured variable.  Finally, data
    is often distorted by contaminating processes, and we need to apply methods
    that are robust against such distortions. That's why we consider the
    robustness of LLC, see \cite{llc}, one of the few causal analysis methods
    that can deal with cyclic models with hidden confounders.  Following a
    theoretical analysis of LLC's robustness properties, we also provide robust
    extensions of LLC. To facilitate reproducibility and further research in
    this field, we make the source code publicly available.

\end{abstract}

\section{Introduction}
We are surrounded by systems, frameworks of smaller parts, that are somehow
connected to produce a larger entity. We encounter those systems in all
natural sciences, physics, biology, and chemistry, but also in other
fields like economy, sociology, medical science, and epidemiology. Further examples are
human-built structures like complex machines, mobile networks, business
processes, or highly advanced simulators. In all those fields, systems often
become so exceedingly complex that it is virtually impossible to understand all
the underlying causal relations between the different constituent parts. An
understanding is, however, of vital importance for the control and maintenance of
those entities. In the field of predictive maintenance, one tries to use anomaly
detection and root cause analysis to automate those tasks, but for this to
work, in particular for the root cause analysis part, it is essential to
understand the underlying mechanisms connecting the different pieces of the
whole structure.

To get some insight, people often consider the most important variables
describing the system as {\em random variables} and investigate their
stochastic relations. This then
leads to networks of random variables, which can be described by various kinds
of graphs like Markov random fields or Bayesian belief networks. However, the
edges in those graphs cannot be interpreted as causal relations. To obtain
information about those causal relations from measurement data, one
needs to apply methods from the field of {\em causality}. In the formulation by Pearl, see
for instance \cite{pearl}, the causality of a system is provided, to a large part, by
its {\em causal graph}. The idea of this causal graph is that here, the
directed edges do indeed describe the causal links between the contained random
variables as well as their causal direction.

One of the goals of causality is thus to obtain the functions describing the
causal relationships between nodes and their parents in causal graphs. Most
techniques presume that all the {\em relevant} random variables are
measured.
Unfortunately, in real life, this is rarely the case. Often, there are random variables that are not
measured but still influence multiple observed variables simultaneously. Those are called {\em hidden
confounders}; see, for instance, \cite{pearl} or \cite{pjs}. To obtain more realistic
models, one thus has to use techniques that still work in the presence of such
hidden confounders.

Another property that is usually presumed for causal models is that they are
{\em acyclic}, meaning that there are no cycles in the causal graph. This is
often justified by the argument that the events of cause and effect are separated
in time, i.e., the cause is happening {\em before} the effect, and that it is thus
physically impossible for the effect to also have some causal influence on the
original cause. However, any system with feedback loops is evidence for the
necessity of cycles in causality. Sometimes, the problem of feedback loops
is handled by considering causal time series graphs, where edges only point
into the future. But those time series discretize time, and it is then
presumed that the measurements have such a high time resolution that we can exclude
interaction between variables at the same discrete time events, i.e. we exclude
instantaneous effects. This approach has two drawbacks. First, it is often
unrealistic to presume the necessary high resolution of measurements, and
second, it leads to a large increase in the number of variables as well as
autocorrelation effects which turns out to be detrimental to the power of the
prediction algorithms. Thus, it is often inevitable to restrict oneself to
methods that can handle cyclic causal relations.

If we require causal methods to be able to handle both hidden confounders and
cycles, only very few remain. One of them is {\em LLC}
(Linear system with Latent confounders and Cycles), see \cite{llc},
which presumes the model to be linear and without self-cycles, takes as input
data the measurements from interventions, and
returns the complete linear model. This is the technique we will investigate
in this paper.

Since we use measurement data to learn the structure of the true underlying
causal model of a system, we have to be prepared for the possibility of low
data quality, i.e., the data could be contaminated with outliers. The
discipline in statistics that addresses the handling of outliers is called {\em
robust statistics}. The focus of this paper is the robustness of the
algorithm LLC.

In more detail, our contributions are as follows: We will first provide an
analysis of the robustness of the original LLC algorithm using
metrics of robustness well established in the robustness literature. Second,
we will extend LLC to become more robust and evaluate those extensions.

\section{Related Work}
Causality is a relatively young branch of statistics. For a full account of its
foundations, see, e.g., \cite{pearl}, \cite{sgs}, and \cite{pjs}. Those treatises also
cover hidden confounders. For a discussion of cyclic systems, however, one has to
look elsewhere. Early accounts can be found in \cite{spirtes1995cyclic},
\cite{pearl1996cyclic}, and \cite{neal2000cyclic}. However, a comprehensive
fundamental treatment of cyclic models was provided only recently; see
\cite{forre2017markov} and \cite{bongers}.

Comprehensive discussions of the foundations of robust statistics can be found,
e.g., in \cite{huber2004robust}, \cite{hampel2011robust}, and
\cite{maronna2019robust}. While Huber's book is the
oldest and most general one, the treatise by Hampel et al.
concentrates on influence functions. What makes the book
by Maronna et al. particularly useful is its emphasis on
freely available implementations in the programming language R, see \cite{r2013r}.

As mentioned above, the focus of this paper is the algorithm LLC. It first
appeared in \cite{eberhardt2010combining} and was further developed in
\cite{llc-f} and \cite{llc}.

\section{Prerequisites}

\subsection{Short Summary of the LLC Algorithm}\label{llcsum}
A very detailed and accessible description of LLC is given in \cite{llc}. Here,
we will provide only an overview of the algorithm, establishing the key notions
and notations, closely following \cite{llc}.

Let's presume that we have a causal system consisting of $d$ observed random variables
$\xx = (x_1, \ldots, x_d)^T$, as well as an equal
amount of unobserved {\em disturbances} (or {\em errors}), given by random
variables $\ee = (e_1, \ldots, e_d)^T$, which together satisfy the following
linear equation:
\begin{equation}\label{llceq}
    \xx = \BB\xx + \ee,
\end{equation}
where $\BB = (b_{ij})_{i,j=1}^d$ is a $(d\times d)$ matrix containing as coefficients the direct causal effects
$b_{ij}$ of the causations $x_j \to x_i$. While we do allow cycles, i.e., for two different
indices $i\ne j$ we allow $b_{ij}$ and $b_{ji}$ to be nonzero at the same time, we don't allow
self-cycles, i.e., $b_{ii} = 0$ for all $i=1,\ldots,d$.
The disturbances $\ee$
satisfy $mean(\ee) = 0$ and the covariance matrix of $\ee$ is referred to by
$\See$.  It can be shown, see \cite{llc}, that the off-diagonal elements of
$\See$ describe the effects of hidden confounders. Thus, the linear causal
model we are considering is described by the pair $(\BB, \See)$, and the goal
of the LLC algorithm is to learn those two matrices from observed data, i.e.,
realizations of $\xx$.

This task is usually not solvable without data from {\em interventions}. An
intervention is an experiment in which the system (\ref{llceq}) is changed such
that a subset $\xx^\prime$ of the observed variables $\xx$, the {\em intervened
variables}, is forced to certain values independent of the state of the other
variables. That means that any causal mechanisms affecting elements of
$\xx^\prime$ that were present in the non-intervened system are severed in the
intervention. This is actually a special type of intervention, a {\em perfect
intervention}, see, e.g., \cite{pearl}. We refer to the experiment without
intervention as the {\em purely observational} experiment. An intervention is
described by the indices $\cJ = \{j_1, \ldots, j_m\}\subset\{1,2,\ldots,d\}, m<d$,
of the intervened nodes. In fact, following \cite{llc}, our notation for an
experiment is $\cE = (\cJ, \cU)$, where $\cU = \{1, \ldots, d\}\setminus \cJ$.
Each experiment $(\cJ, \cU)$ comes with two belonging matrices $\JJ:=diag(\cJ),
\UU:=diag(\cU)$, where $diag(\cI), \cI \subset\{1,\ldots,d\}$, describes the matrix that is zero everywhere
except for the diagonal elements whose row indices are contained in $\cI$,
and those diagonal elements are set to one. Note, that $\UU + \JJ = \II$, where $\II$ is the $d$-dimensional identity matrix.
The index sets $\cJ$ and $\cU$ will also be used in a convenient notation for
submatrices. Namely, the matrix $M_{\cI_1\cI_2}$ with index sets $\cI_1,
\cI_2\subset\{1,\ldots,d\}$ is defined as the submatrix of $M$ that is the intersection of
all the rows of $M$ with indies in $\cI_1$ and all the columns of $M$ with indices in
$\cI_2$.

We usually consider a set of experiments $\{\cE_k | k=1,\ldots, K\}$ with
pertinent matrices $\{(\JJ_k, \UU_k) | k = 1,\ldots, K\}$.  In an experiment,
the intervened variables are replaced by independent standard normal random
variables $\cc$. Thus, LLC's defining equation (\ref{llceq}) in the presence of
an intervention $(\cJ_k, \cU_k)$, is given by (see \cite{llc}):
\begin{equation}
    \xx = \UU_k\BB\xx + \UU_k\ee + \JJ_k\cc\label{llceqExp}
\end{equation}
where $\cc \sim \cN(\zero, \II)$.
The algorithm LLC requires the system to be {\em weakly stable}, which means
that the matrix $\BB$ has to satisfy that $\II-\UU_k\BB$ is invertible for each
experiment.

Using interventions, we can build an estimator $\hat\BB$ of the matrix $\BB$
from the observations $\xx$ as follows. Let's first consider the covariance
matrix $\Cxk$ of $\xx$ in the experiment $\emt$. The coefficient $c^k_{ui}$ of
$\Cxk$ with $i\in\cJ$ and $u\in\cU$ is called the {\em total causal effect} of
the intervened variable $x_i$ on the non-intervened variable $x_u$ in the
experiment $\emt$, written as
$t(x_i\rightsquigarrow x_u||\cJ_k)$. This name makes sense because if
forced changes of $x_i$ are correlated with observations of $x_u$, then there
must be a total, i.e.\ possibly via multiple causal paths, effect of $x_i$ on $x_u$. This
correlation cannot be due to a causal effect from $x_u$ to $x_i$, or originate from
another variable $x_c$ that causally influences both $x_i$ and $x_u$,
$x_i\leftarrow x_c \to x_u$, because $x_i$ is an intervened variable, i.e.\ all
causal mechanisms into $x_i$ are severed. Using those total effects, we can
establish the following crucial linear constraints for the direct effects
$b_{rs}$ (see \cite{llc} for a proof):
\begin{equation}
    t(x_i\rightsquigarrow x_u||\cJ_k) = b_{ui} +
    \sum_{u^\prime\in\cU_k\setminus\{u\}}
    t(x_i\rightsquigarrow x_{u^\prime}||\cJ_k)b_{uu^\prime}.\label{pathAna}
\end{equation}
In words, under the experiment $\emt$, the total causal effect of an intervened
variable $x_i$ on a non-intervened variable $x_u$ is the direct effect of $x_i$
on $x_u$ plus the sum of the total effects on all the other non-intervened
variables $x_{u^\prime}$ times their direct effects on $x_u$. Note that only some of the coefficients
of $\Cxk$ are total effects. Namely, they are contained in the submatrices $\TT^k$:
\begin{equation}\label{ttk}
    \TT^k := (\Cxk)_{\cU_k\cJ_k} = (\Cxk)_{\cJ_k\cU_k},
\end{equation}
where the equality holds because covariance matrices are symmetric.

The constraints (\ref{pathAna}) are actually all we need to construct the
estimator $\hat\BB$. We collect all those constraints from all experiments
into one large system of linear equations in the unknowns $b_{ij}$:
\begin{equation}
    \bt = \TT\bb, \label{ttb}
\end{equation}
where $\bt$ contains all the total effects from the LHS of (\ref{pathAna}),
$\bb$ is the (row-wise) flattened version of $\BB$, but without the diagonal
elements (they are always zero, see above), and $\TT$ contains in each row the total
effects according to the RHS of (\ref{pathAna}). Note, that since in each of
the constraints (\ref{pathAna}) all the $b_{us}$ on the RHS belong to the
$u$-th row of $\BB$, i.e. all direct effect are those into $x_u$, we can
arrange the rows of $\TT$ such that $\TT$ has a block-diagonal structure (note, that those blocks, as well as $\TT$, are usually not quadratic).

In \cite{llc}, it is shown that if the experiments $\emt$ satisfy the so-called
{\em pair conditions}, then the system (\ref{ttb}) has full column rank and can be
solved using standard methods.

After an estimate $\BBh$ of $\BB$ is available, it can be used to estimate $\See$: From
(\ref{llceq}) we have:
\begin{equation}
    \begin{split}
        \xx &= \BB\xx + \ee\\
        \ee &= (\II - \BB)\xx\\
        \See &= (\II-\BB)\CC_\xx^0(\II-\BB)^T,\label{SeeEq}
    \end{split}
\end{equation}
where $\CC_\xx^0$ is the covariance of $\xx$ for the purely observational
experiment. Whence, having an estimation of $\BB$ and of $\CC_\xx^0$, we can
compute an estimation $\Seeh$ of $\See$. This completes the LLC algorithm.

\subsection{Robustness Analysis}
Robustness analysis is about measuring how strongly estimators are influenced by
outliers, i.e.\ contaminations of the dataset. In this section, we will
establish a minimum of notations for robustness analysis that are
needed for this paper, following the standard literature. For more details, see
\cite{huber2004robust}, \cite{hampel2011robust}, and \cite{maronna2019robust}.

Consider a given statistical model $\cM$ and a sample $P$ of size $n$ from
$\cM$. Next, assume that a fraction of $P$ of size $m, m<n$, is replaced by a
dataset $C$ of size $|C|=m$, thus creating a
new dataset $D, |D| = n$. Then, we call
$C$ a {\em contamination}, or the {\em outliers}, and $D$ a {\em contamination of $P$}.
The ratio $\ve(D):=\frac{m}{n}$ is the {\em contamination
rate} of $D$.

To assess the robustness of estimators against contamination, we need a measure
of robustness. There are various such metrics available in the literature. We
will focus on the {\em breakdown point} ({\em BP}), see, e.g.,
\cite{maronna2019robust}. Adapted to our situation, the BP of the
LLC estimators $\llce$ is the largest number $\bp$ with the property that
for any given finite sample $P$ of the model $(\BB, \See)$,
$\llce$ are bounded on all contaminations $D$ of $P$ 
with $\ve(D)\le\bp$. I.e., for any given sample $P$ of $(\BB, \See)$,
there will be no sequence $s=(D_i)_{i=1}^\infty$ of contaminations of $P$ satisfying
$\ve(D_i)\le\bp, i=1,2,\ldots$ with the estimators diverging to infinity on $s$.

\subsection{The MCD Algorithm}\label{secMcd}
The {\em Minimum Covariance Determinant} ({\em MCD}) estimator, see \cite{mcd}, is one of the
most popular multidimensional location and scatter estimators with a high BP.
The idea is to select from the full dataset $D$ of size $n$  only a particular
subset $E$ of size $h$, chosen in such a way that we can have reason to hope
$E$ doesn't contain any outliers, discard the rest, which thus should contain
all the outliers, and then proceed with this subset $E$. The size $h$ of $E$ is
a configuration parameter. The subset is chosen as the one set among all the
subsets of size $h$ that minimizes the determinant of its sample covariance
matrix. More formally, let $E^\prime$ be a dataset of size $h$ and
$S(E^\prime)$ be the sample covariance matrix of $E^\prime$. Then
\begin{align*}
    E &:= \argmin_{E^\prime\in\cD_h} \{\det(S(E^\prime))\}\\
    \cD_h &:=  \{E^\prime|E^\prime\subset D \wedge |E^\prime| = h\}.
\end{align*}
The location estimator of MCD is then simply the mean of $E$, and the scale
estimator is the sample covariance matrix $S(E)$, multiplied with some constant
to make it consistent in the Gaussian case.

Larger values of $h$ increase the efficiency of MCD and lower values the
robustness.  If $d$ is the dimension of the sample points, then it can be
shown, see \cite{daviesMcdBp}, that, using the bracket notation
$\lfloor\ldots\rfloor$ to indicate the ``floor'' function, for $h =
\lfloor\frac{n + d + 1}{2}\rfloor$ the BP will attain its maximum of
\begin{equation*}
    BP_{MCD}^{max} = \frac{m^\ast_{max}}{n},\quad m^\ast_{max}:=\left\lfloor\frac{n-d}{2}\right\rfloor.
\end{equation*}
Thus, $BP_{MCD}^{max}\to\frac{1}{2}, n\to\infty$.

MCD is available in various programming languages. The language R, as described in \cite{r2013r},
has even multiple implementations. We have used the package {\tt rrcov}, and
therein the function {\tt CovMcd}, see for instance \cite{rrcov}. The
most important parameter of {\tt CovMcd} is {\tt alpha} which sets
$h/n$. Its values are restricted to the interval $[0.5, 1]$, and the default is
set to $0.5$.

\subsection{Gamma Divergence Estimation}\label{secGde}
In this paper, we will also use the {\em Gamma Divergence Estimation} ({\em
GDE}) method, see \cite{fujiGamma}, for performing robust estimation. As the
name suggests, GDE is a divergence, i.e.\ some form of distance notion between
probability distributions, see for instance \cite{AyIG}, and can thus be
understood in an appealing information-geometric setting, see for example
\cite{egKom}. In this interpretation, a parameter estimator for a given model
family $\cM$ is a projection, in the space $\cP$ of probability distributions,
to the distributions belonging to the family $\cM$, which can itself be
understood as a submanifold in $\cP$. In this view, a finite dataset $D$
is described as a distribution that is the normalized sum of the
Dirac delta distributions of all the points in $D$, which is sometimes called
the {\em empirical probability density function} ({\em EPDF}). This $D\in\cP$ is
positioned outside of the model family submanifold $\cM$ and we would like to
find its ``projection'' to $\cM$, i.e., the point in $\cM$ that lies nearest to $D$ with respect to a given
divergence. Note that this is just an intuitive description. For a mathematically
rigorous account of divergences and estimators in an information-geometric
setting see \cite{AyIG}.

There is a large variety of divergences to choose from, each resulting in a
different parameter estimator. The task is to define some desired
requirements, like consistency, efficiency, or robustness,
and then to find those divergences that result in estimators that fulfill those
requirements best. Comprehensive studies of this task can be found in, e.g.,
\cite{basu}, which also describes the popular {\em density power divergences},
or in the more recent book \cite{egKom}, which also contains a detailed account of GDE.

Following \cite{fujiGamma}, we give a short overview of GDE. Let
\begin{equation*}
    \cM = \{f(x; \theta)| f \mbox{ is a density in } x\in Q,\;\theta\in\Theta\}
\end{equation*}
be some model family, where $Q$ is some probability measure space and $\Theta$
some parameter space. Then, for $\gamma\in\R, \gamma>0$, the
{\em $\gamma$ divergence} $D_\gamma(h_1, h_2)$  of two densities $h_1, h_2$ is
defined as:
\begin{multline*}
    D_\gamma(h_1, h_2) := \frac{1}{\gamma(1+\gamma)}\log\int h_1^{1+\gamma}\,dx\;+\\
                          -\frac{1}{\gamma}\log\int h_1(x)h_2(x)^\gamma dx\;+
                          \frac{1}{1+\gamma}\log\int h_2(x)^{1+\gamma} dx.
\end{multline*}
Now, let
\begin{equation*}
    g(x) = (1-\ve)f(x;\theta) + \ve c(x),
\end{equation*}
be a contamination of $f(x;\theta)\in\cM$, where $\ve \ge 0$ is the contamination rate and
$c(x)$ is some contamination density. To create a robust estimator for $\cM$, $D_\gamma$
needs to satisfy $D_\gamma(g(x), f(x; \theta^\prime)) \approx D_\gamma(f(x;\theta),
f(x;\theta^\prime))$ for any $f(x;\theta), f(x;\theta^\prime)\in\cM$, thus mostly
disregarding the contamination $c(x)$. The value of $\gamma$ is a configuration
parameter that has to be set by the user. Roughly, larger values of $\gamma$
result in more robust but less efficient estimators. The optimal choice depends on the
situation, but values in the interval $[0.2, 0.4]$ have quite generally worked
well for us. Finding the best $\gamma$ for a given scenario is an
open problem.

In practice, we only have a finite sample $\{x_i\}_{i=1}^n$ of $g(x)$. Then, the
divergence is approximated using the EPDF $\bar g(x)$ of $g(x)$:
\begin{equation*}
    D_\gamma(\bar g(x), f(x;\theta)) = -\frac{1}{\gamma}\log\left(\frac{1}{n}\sum_{i=1}^n f(x_i;\theta)^\gamma\right) +\\
        \frac{1}{1+\gamma}\log\int f(x;\theta)^{1+\gamma}dx + const.
\end{equation*}
The computation of the BP of GDE is a bit involved, but in \cite{fujiGamma} the
authors state that the BP can be regarded as $1/2$ when the parametric density
is normal.

We use GDE for the robust estimation of the multivariate covariance matrices
$\CC^k_\xx$, presuming that $\xx$ is approximately normally distributed. In this
case, GDE can be computed using the {\em concave-convex procedure} (CCP), see
\cite{ccp}. The iteration is described in \cite{egKom}, but no implementation
has been available. Thus, we make the source code of our implementation freely
available, see section (\ref{secPa}) for details.

\subsection{Distance Metric for Covariance Matrices}\label{secMetric}
There is a multitude of possible distance measures in the space of covariance
matrices. See, e.g., \cite{arsigny2006log} or \cite{dryden2009non}. However, in
this paper, we will stick to a classic measure, the {\em Frobenius norm}, when
investigating the robustness of the LLC estimators. This is in line with the
majority of publications, see e.g. \cite{shrink} and \cite{shrinkNL}. We will
also use the {\em relative Frobenius error} ({\em RFE}) which is the ratio of the
Frobenius norm of the estimation residual and the Frobenius norm of the true
value.

\section{Theoretical Analysis}
Here we will investigate the theoretical robustness properties of LLC.

\subsection{Different Types of Contamination}
The estimation of the model $(\BB, \SSigma_\ee)$ is based on the estimation of
the vector $\bt$ from (\ref{ttb}), which in turn is obtained from submatrices $\TT^k$ of the
experimental covariance matrices $\Cxk$ as given in (\ref{ttk}). Thus, we
need to establish which data the matrices $\Cxk$ depend on, to understand what origins
of contamination we have to deal with.

Using (\ref{llceqExp}) and recalling that we presume weak stability, we have 
\begin{equation*}
    \xx = (\II - \UU_k\BB)^{-1}(\UU_k\ee + \JJ_k\cc).
\end{equation*}
which allows us to write the covariance of the measurements $\xx$ in the $k$th
experiment as follows (we use $\SSigma_\cc:= \cov(\cc)$):
\begin{align}
    \Cxk &= (\II - \UU_k\BB)^{-1} \cov(\UU_k\ee + \JJ_k\cc) (\II-\UU_k\BB)^{-T}\nonumber\\
        &= (\II - \UU_k\BB)^{-1} (\UU_k\SSigma_\ee\UU_k + \JJ_k\SSigma_\cc\JJ_k) (\II-\UU_k\BB)^{-T}.\label{covx2}
\end{align}
This means that the $\Cxk$ depends on both $\ee$ and $\cc$. I.e., we can
consider three types of contamination: a contamination of $\ee$, $\cc$, or $\xx$.

First, note that LLC is not bound to any constraints on the distribution of
$\ee$, which means that the estimation of $\BB$ is unaffected by contaminations
of $\ee$ (see also the proof of Lemma \ref{lemmaIntC}). However, for the
estimation of $\SSigma_\ee$, contaminations of $\ee$ are, of course, relevant.

Next, we turn to contaminations of the random vector $\cc$ of interventions.

\begin{lemma}\label{lemmaIntC}
    The matrix $\TT^k$, in general, depends on the distribution of the random vector $\cc$.
\end{lemma}
\begin{proof}
    For better insight, we derive a formula that describes the dependence of $\TT^k$ on $\cc$.
    First, we note that
    $(\Cxk)_{\cU_k\cJ_k} = (\UU_k\Cxk\JJ_k)_{\cU_k\cJ_k}$, i.e. it
    suffices to consider $\UU_k\Cxk\JJ_k$. Combining (\ref{covx2}) and
    (\ref{ttk}), we then  obtain:
    \begin{equation}
            \TT^k = \Big(\UU_k(\II - \UU_k\BB)^{-1} (\UU_k\SSigma_\ee\UU_k + \JJ_k\SSigma_\cc\JJ_k)
                         (\II-\UU_k\BB)^{-T}\JJ_k\Big)_{\cU_k\cJ_k}\label{longttk}.
    \end{equation}
    Here, we first focus on the last two factors, $(\II - \UU_k\BB)^{-T}\JJ_k$:
    \begin{align}
        (\II-\UU_k\BB)(\II-\UU_k\BB)^{-1} &= \II\nonumber\\
        \JJ_k(\II-\UU_k\BB)(\II-\UU_k\BB)^{-1} &= \JJ_k\II\nonumber\\
        \JJ_k(\II-\UU_k\BB)^{-1} - \JJ_k\UU_k\BB(\II-\UU_k\BB)^{-1} &= \JJ_k\nonumber\\
        \JJ_k(\II-\UU_k\BB)^{-1} &= \JJ_k    \label{jkthetaisjk}\\
        (\II-\UU_k\BB)^{-T}\JJ_k &= \JJ_k    \label{thetatjk},
    \end{align}
    where (\ref{jkthetaisjk}) follows because $\JJ_k\UU_k=\zero$.
    Next, (\ref{longttk}) and (\ref{thetatjk}) obtain:
    \begin{align}
        \TT^k &= \Big(\UU_k(\II - \UU_k\BB)^{-1} (\UU_k\SSigma_\ee\UU_k + \JJ_k\SSigma_\cc\JJ_k)\JJ_k
                     \Big)_{\cU_k\cJ_k}\nonumber\\
              &= \Big(\UU_k(\II - \UU_k\BB)^{-1} (\zero + \JJ_k\SSigma_\cc\JJ_k)\Big)_{\cU_k\cJ_k}\label{zeroStep}\\
              &= \Big(\UU_k(\II - \UU_k\BB)^{-1} \JJ_k\SSigma_\cc\JJ_k\Big)_{\cU_k\cJ_k},\label{finalttk}
    \end{align}
    where (\ref{zeroStep}) holds because $\UU_k\JJ_k = \zero$ and $\JJ_k\JJ_k=\JJ_k$.
    This reconfirms our claim above that the estimation of $\BB$ is independent of the distribution of $\ee$.

    Equation (\ref{finalttk}) strongly indicates that this lemma holds. To be precise, we have to provide an example showing
    that $\TT^k$ is not a constant function of $\SSigma_\cc$.
    Such an example is given by:
    \begin{equation*}
        \BB = \left(\begin{matrix}0 & 1 \\ 1 & 0\end{matrix}\right),
    \end{equation*}
    and the experiment
    \begin{equation*}
        \JJ_k = \left(\begin{matrix}1&0\\0&0\end{matrix}\right), \quad
        \UU_k = \left(\begin{matrix}0&0\\0&1\end{matrix}\right).
    \end{equation*}
    A substitution of those matrices into (\ref{finalttk}) shows that changes
    in $\SSigma_\cc$ can affect $\TT^k$. 
\end{proof}

Finally, it is obvious from (\ref{ttk}) that contaminations of $\xx$ affect the
estimation of $\bt$.

Thus, we have shown that the estimation of $\BB$ can be influenced by
contaminations of $\cc$ and $\xx$, but not of $\ee$. The estimate
$\hat\SSigma_\ee$ is obtained via the equation
\begin{equation}
    \hat\SSigma_\ee = (\II - \hat\BB)\hat\CC_\xx^0(\II - \hat\BB)^T,\label{estSigmaE}
\end{equation}
where $\hat\CC_\xx^0$ is the estimator of the covariance of $\xx$ in the purely
observational experiment. That means, that $\hat\SSigma_\ee$ can be affected by
all three types of contaminations, possibly indirectly through the estimate
$\hat\BB$. Again, to be precise, one would have to give an example, showing
that (\ref{estSigmaE}) is not a constant function in $\xx$, $\ee$, or $\cc$,
but this works analogously to the proof of Lemma (\ref{lemmaIntC}) and will
not be repeated here.

To summarize, we have the following proposition:
\begin{proposition}\label{propContTypes}
    The estimator $\BBh$ is not affected by contaminations of $\ee$ but is
    affected by contaminations of $\cc$ and $\xx$. The estimator $\Seeh$ is
    affected by all three types of contaminations.
\end{proposition}

\subsection{Breakdown Point of the Estimation of \texorpdfstring{$\BB$}{B}}
As described in Section \ref{llcsum}, the estimator $\BBh$ proceeds as follows.
For each experiment $\emt$, we estimate the covariance
matrix $\Cxk$, which contains certain total effects in the submatrix $\TT^k$ in (\ref{ttk}). Those total effects are
then used in the linear constraints (\ref{pathAna}), which are combined, over
all experiments, into (\ref{ttb}). Then, it follows from the pair condition that the
inhomogeneity $\bt$ in (\ref{ttb}) contains all $d^2-d$ possible total effects.

Since the matrix $\TT$ contains only total effects and zeros, we have that
$\TT$ is a function of the inhomogeneity $\bt$, $\TT = \TT(\bt)$. Thus,
\begin{align}
        \bt &= \TT(\bt)\bb,\label{solve4b}\\
        \bb &= \TT^{-1}(\bt)\bt\label{solve4bMPI},
\end{align}
where $\TT^{-1}(\bt)$ is the {\em Moore-Penrose pseudoinverse}. It should be
stressed that, even though the structure of (\ref{solve4b}) is similar to that
of linear regression problems, our situation is different, some of the largest
differences being that the matrix is a function of the inhomogeneity and that
our problem is unsupervised. Note that
\begin{equation*}
    f(\bt):=\TT^{-1}(\bt)\bt
\end{equation*}
is a function of the total effects $\bt$ and that
this function is nonlinear. Furthermore, note, that different experiments can
contribute values for the same total effects, and that those values are likely
to be different due to finite data. Thus, $\bt$ can contain different versions
of the same total effect. As a result, (\ref{solve4b}) doesn't have to be true
anymore, but the Moore-Penrose pseudoinverse in (\ref{solve4bMPI}) still gives
an approximation of $\bb$ by using the projection of $\bt$ to
$\range(\TT(\bt))$, see \cite{llc}, p.3407.

The LLC algorithm uses the standard estimator for covariance, called the {\em
Sample Covariance Matrix} ({\em SCM}), to estimate $\Cxk$. It is well known
that SCM is non-robust with a BP of zero. A single unbounded
element of the sample makes SCM unbounded. Thus, since in LLC, the total effects
$\bt$ are coefficients of those estimated covariance matrices $\Cxk$, if the
function $f(\bt)$ is unbounded on at least one of the paths in which
contaminations can drive $\bt$ to infinity, the estimator $\BBh$ has BP zero,
too.
\begin{lemma}\label{lemmaFUnbAtInf}
    Contaminations of $\xx$ can drive $f(\bt)$ to infinity.
\end{lemma}
\begin{proof}
    We prove this by giving an example.
    Consider, e.g., a simple two-node
    model. The two possible singular interventions result in the following two
    equations:
        \begin{equation}\label{twoNodeEx}
            \begin{split}
                t_{12} &= b_{12}\\
                t_{21} &= b_{21},
            \end{split}
        \end{equation}
    i.e., $\TT(\bt) = \II$ and $f(\bt) = \bt$ is unbounded in all directions.
\end{proof}
The reference implementation of LLC in \cite{llc} also provides the option of
using $L_2$-regularization to solve (\ref{solve4b}). However, regularization
doesn't help:

\begin{lemma}
    The $L_2$-regularized estimator $\BBh$ is unbounded.
\end{lemma}
\begin{proof}
    We use the same simple example (\ref{twoNodeEx}) as in the proof of
    Lemma \ref{lemmaFUnbAtInf}. Let $L(\bb)$ denote the loss with
    regularization parameter $\lambda > 0$:
    \begin{equation*}
        L(\bb) = \|\TT(\bt)\bb - \bt\|^2 + \lambda\|\bb\|^2.
    \end{equation*}
    Then, a straightforward application of the technique of completing the
    square gives:
    \begin{align*}
        L(\bb) &= \|\TT(\bt)\bb - \bt\|^2 + \lambda\|\bb\|^2\\
               &= \|\bb - \bt\|^2 + \lambda\|\bb\|^2\\
               &= \langle\bb - \bt, \bb - \bt\rangle + \lambda\langle\bb, \bb\rangle\\
               &= (1+\lambda)\left(\bb^2 - \frac{2}{1+\lambda}\langle \bb, \bt\rangle + \frac{\bt^2}{1+\lambda}\right)\\
               &= (1+\lambda)\left(\left(\bb - \frac{\bt}{1+\lambda}\right)^2 - \frac{\bt^2}{(1+\lambda)^2} + \frac{\bt^2}{1+\lambda}\right)\\
               &= (1+\lambda)\left(\bb - \frac{\bt}{1+\lambda}\right)^2 + \frac{\lambda}{1+\lambda}\bt^2,
    \end{align*}
    which, noting $\lambda > 0$, is a convex function in $\bb$ with minimum at
    $\frac{\bt}{1+\lambda}$. Since $\frac{\bt}{1+\lambda}$ in unbounded,
    the regularized estimator $\BBh$ is unbounded, too. 
\end{proof}

We have thus shown, that both the regularized and unregularized estimations of
$\BB$ have a BP of zero. The origin of this nonrobustness is the nonrobustness of
the SCM estimation of the covariance matrices $\Cxk$. However, there is yet
another source of non-robustness. In fact:

\begin{lemma}\label{lemmaSingB}
The function $f(\bt)$ can have singularities.
\end{lemma}
\begin{proof}
    We show this again by giving an example. Let's consider a three-node model
    with three single-node experiments, i.e., each experiment intervenes on
    exactly one node. Then, the matrix $\TT(\bt)$ has the following block-diagonal
    structure:
    \begin{equation*}
        \TT(\bt) = \left(\begin{matrix}
                   1 & t_{32} &      0 &      0 &      0 &      0 \\
              t_{23} &      1 &      0 &      0 &      0 &      0 \\
                   0 &      0 &      1 & t_{31} &      0 &      0 \\
                   0 &      0 & t_{13} &      1 &      0 &      0 \\
                   0 &      0 &      0 &      0 &      1 & t_{21} \\
                   0 &      0 &      0 &      0 & t_{12} &      1 
        \end{matrix}\right),
    \end{equation*}
    where we use the notation $t_{ui}$ to indicate the total effect of the
    intervened node $x_i$ on the non-intervened one $x_u$.  Here, the columns belong to the
    coefficients $b_{12}, b_{13}, b_{21}, b_{23}, b_{31}, b_{32}$, resp, and
    the rows belong to the total effects $t_{12}, t_{13}, t_{21}, t_{23},
    t_{31}, t_{32}$, resp. 

    Note, that the inverse matrix $\TT^{-1}(\bt)$ has also block diagonal structure,
    with each block being the inverse of the pertinent block in $\TT(\bt)$.
    Whence, each of the diagonal $(2\times 2)$ blocks can be used to create a
    singularity, e.g., in the first block, by setting $t_{23} = t_{32}^{-1}$.
    The situation is a bit similar to the problem of multicollinearity in linear regression.
    But, to be precise, we have to show that not only $\TT^{-1}(\bt)$ becomes singular, but
    $\TT^{-1}(\bt)\bt$ does so, too. To this end, consider, for instance, the first block.
    Using the formula for the two-dimensional matrix inverse, we get:
    \begin{equation*}
        \left(\begin{matrix}b_{12}\\b_{13}\end{matrix}\right) = \frac{1}{1-t_{32}t_{23}}
            \left(\begin{matrix}
                1       & -t_{32}\\
                -t_{23} & 1  
        \end{matrix}\right)
        \left(\begin{matrix}t_{12}\\t_{13}\end{matrix}\right).
    \end{equation*}
    Thus, letting $t_{23}\to 1/t_{32}$, the denominator converges to zero.
    Now, it can easily happen that $t_{12} - t_{32}t_{13}$ is bounded away
    from zero since there is no deterministic connection between those total effects. Thus, we
    obtain a singularity of $f(\bt)$ for finite $\bt$. 
\end{proof}

Thus, we have proven the following proposition:
\begin{proposition}\label{propBBhat}
    The estimator $\BBh$ is not robust, having a BP of zero. Even using an
    estimator $\Cxkh$ with nonzero BP would, in general, still result in an
    estimator $\BBh$ with BP zero.
\end{proposition}

\subsection{Breakdown Point of the Estimation of \texorpdfstring{$\See$}{the disturbance covariance matrix}}
The situation for the estimator $\Seeh$ is more straightforward: Using
(\ref{SeeEq}), it is given by:
\begin{equation}
    \Seeh = (\II - \hat\BB)\hat\CC_\xx^0(\II - \hat\BB)^T.\label{SigmaEObs}
\end{equation}
We then have:
\begin{proposition}\label{propSeeh}
    The estimator $\Seeh$ has a BP equal to zero. Even using an estimator
    $\Cxkh$ with nonzero BP would, in general, still result in an estimator
    $\Seeh$ with BP zero.
\end{proposition}
\begin{proof}
    The estimator (\ref{SigmaEObs}) combines the estimators $\hat\CC_\xx^0$ and
    $\hat\BB$. Note, that $\hat\CC_\xx^0$ and $\hat\BB$ are estimated from
    different experiments, i.e., their contaminations are independent. Whence,
    and since $\hat\CC_\xx^0$ is non-robust with BP zero and $\Seeh$ is linear
    in $\hat\CC_\xx^0$, it follows that $\Seeh$ has a BP of zero, too.

    Furthermore, because of Proposition \ref{propBBhat} and because $\BBh$ and
    $\hat\CC_\xx^0$ are estimated using different experiments, even using
    robust estimators with nonzero BP will, in general, still result in a zero
    BP of $\Seeh$. 
\end{proof}

\section{Practical Analysis}\label{secPa}
The analysis of theoretical robustness properties provides us with insight into
the general, asymptotic features of estimators. To also better understand the
behavior in finite settings, we will now consider two extensions
of the LLC estimator to improve the empirical robustness of $\llce$. In both
cases, the approach consists of replacing the SCM version of
$\Cxkh$ with more robust covariance estimation methods, namely MCD and GDE, as
described in Section \ref{secMcd} and Section \ref{secGde}, resp.
Note, that even though MCD and GDE have nonzero BPs, this doesn't translate
into nonzero BPs of $\llce$ because of Proposition \ref{propBBhat} and Proposition
\ref{propSeeh}.

\begin{figure}[h!]
    \centering
    \includegraphics[width=10cm]{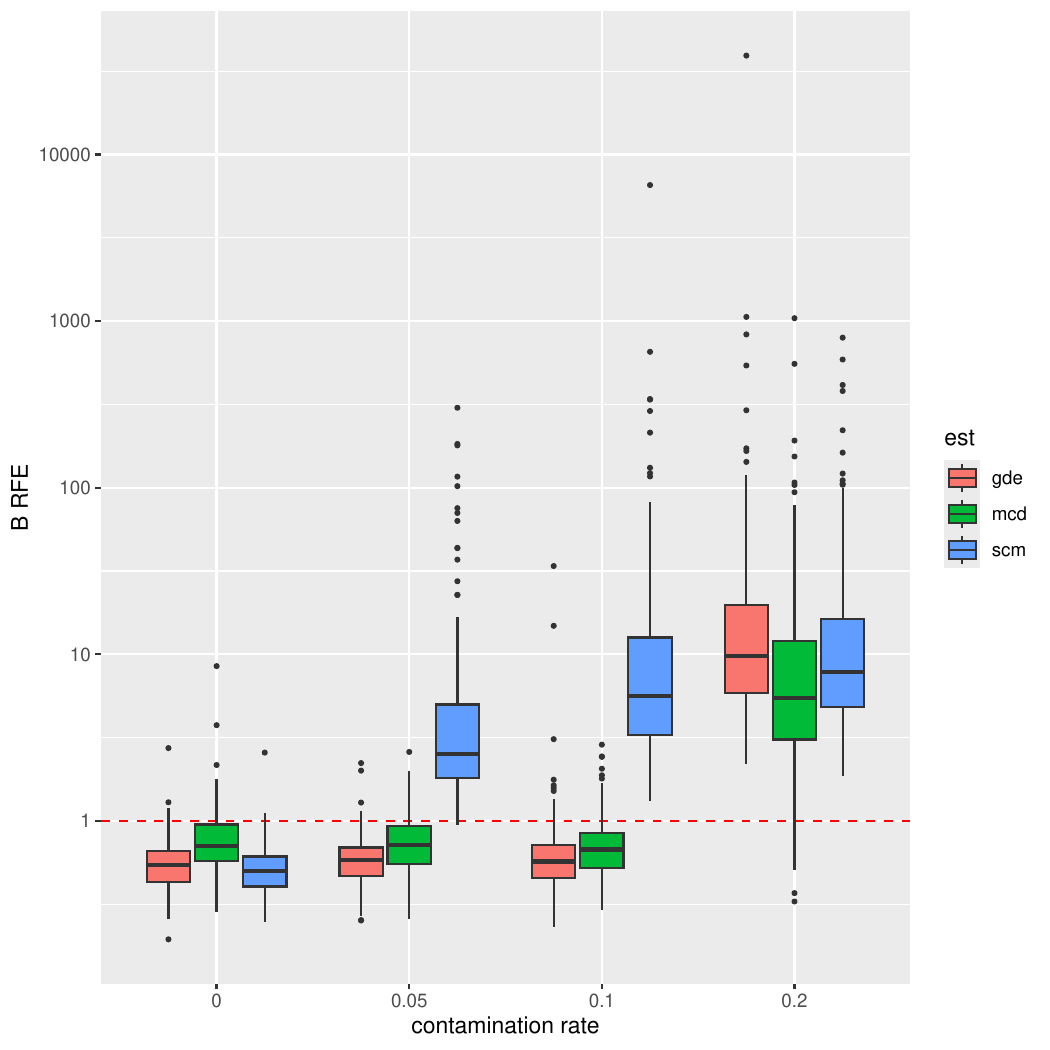}
    \caption{Logarithmic RFE for $\BB$ as a function of the contamination rate for the
    default implementation (SCM) and the two robustified versions (MCD,
    GDE)}\label{figBoxB}
\end{figure}

\begin{figure}[h!]
    \centering
    \includegraphics[width=10cm]{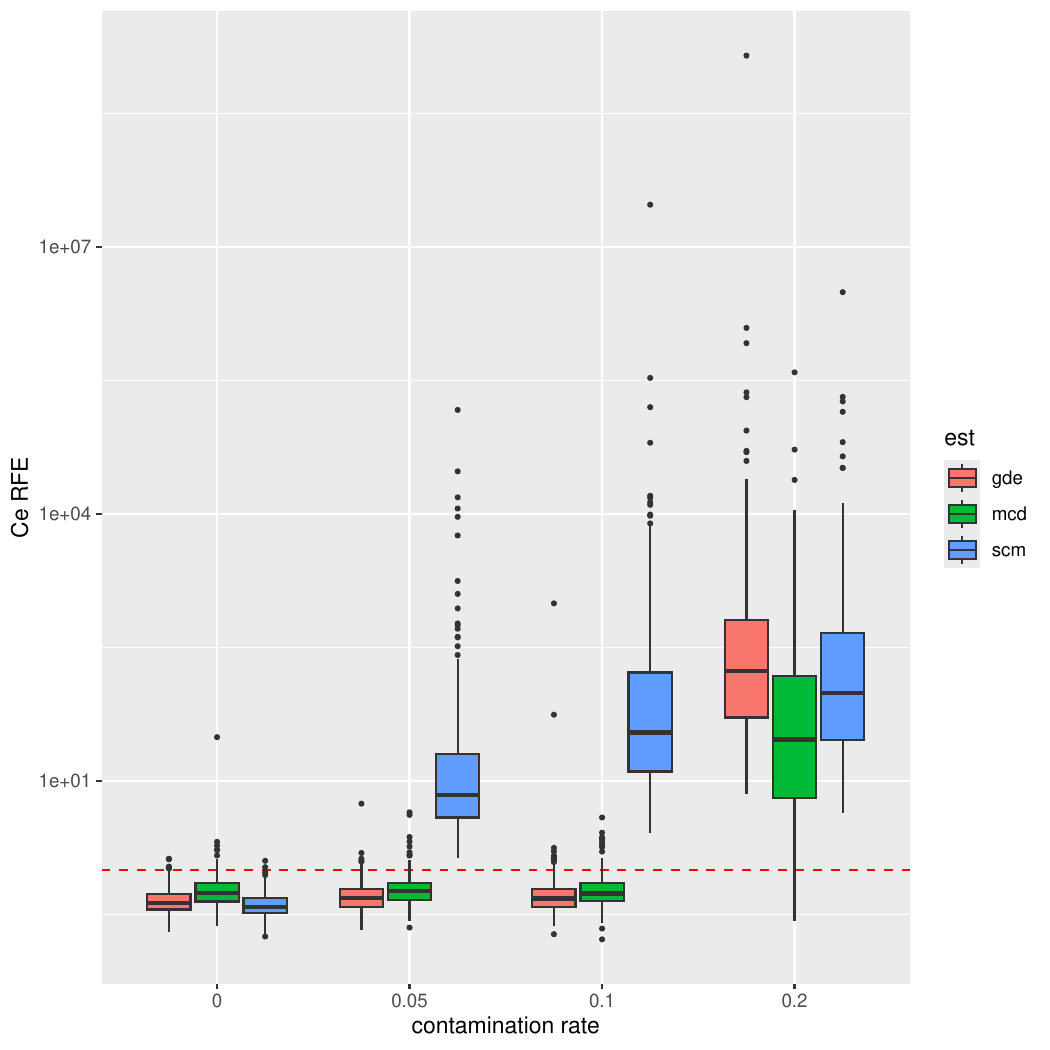}
    \caption{Logarithmic RFE for $\See$ as a function of the contamination rate for the
    default implementation (SCM) and the two robustified versions (MCD,
    GDE)}\label{figBoxCe}
\end{figure}

For the evaluation of those robust variants of LLC, we have to apply them to
data that we know the causal ground truth of. Unfortunately, there are only
very few real-world datasets that satisfy this condition. And they are
certainly not available with known contaminations and in amounts necessary for the statistical evaluations
conducted here. That is why we confine ourselves to synthetic data. We will consider 200 randomly generated causal systems with five
nodes each. The probability of an edge is set to 0.3, as is that of a
confounder. For each model, we simulated six experiments, one is the purely
observational experiment, joined by one singular perfect intervention experiment
for each node. For each experiment, we create a sample of size 200, which will
contain randomly generated contaminations of $\xx$. For more details about the
configurations, see our R implementation on GitHub (the repository will
be made publicly available after the publication of the paper).

The three estimators considered are the LLC estimators $\llce$ using for
the estimation of $\Cxk$ the estimators SCM, MCD, and GDE.
For the comparison we use the relative Frobenius error (RFE) as discussed in Section \ref{secMetric}. First, we
examine data without any contamination ($\ve=0$) and record for all three
estimators the median and MAD (Median Absolute Deviation), taken over the 200 random models, of the RFE.
The result is given in Table \ref{tableZeroCont}.
Next, the same procedure was applied to contamination rates $\ve=0.05$ and $\ve=0.1$ with the results
in Table \ref{table0.05Cont} and \ref{table0.1Cont}.

\begin{table}[h]
    \centering
    \caption{Median and MAD of RFEs for $\ve=0$}
    \begin{tabular}{rrrrr}
        \toprule
        \multirow{2}{*}{Estimator} &
            \multicolumn{2}{c}{$\BBh$ RFE} &
            \multicolumn{2}{c}{$\Seeh$ RFE} \\
            & Median    & MAD           & Median        & MAD           \\
            \midrule
        SCM & 0.50      &  0.15         &   0.38        &    0.11       \\
        MCD & 0.71      &  0.27         &   0.55        &    0.17       \\
        GDE & 0.54      &  0.17         &   0.43        &    0.12       \\
        \bottomrule
    \end{tabular}
    \label{tableZeroCont}
\end{table}

\begin{table}[h]
    \centering
    \caption{Median and MAD of RFEs for $\ve=0.05$}
    \begin{tabular}{rrrrr}
        \toprule
        \multirow{2}{*}{Estimator} &
            \multicolumn{2}{c}{$\BBh$ RFE} &
            \multicolumn{2}{c}{$\Seeh$ RFE} \\
            & Median    & MAD           & Median        & MAD           \\
            \midrule
        SCM &  2.52     &   1.46        &      6.94     &       5.83    \\
        MCD &  0.72     &   0.28        &      0.58     &       0.18    \\
        GDE &  0.58     &   0.17        &      0.49     &       0.16    \\
        \bottomrule
    \end{tabular}
    \label{table0.05Cont}
\end{table}

\begin{table}[h!]
    \centering
    \caption{Median and MAD of RFEs for $\ve=0.1$}
    \begin{tabular}{rrrrr}
        \toprule
        \multirow{2}{*}{Estimator} &
            \multicolumn{2}{c}{$\BBh$ RFE} &
            \multicolumn{2}{c}{$\Seeh$ RFE} \\
            & Median    & MAD           & Median        & MAD           \\
            \midrule
        SCM &  5.63     &   4.72        &     35.02     &      42.86    \\
        MCD &  0.68     &   0.24        &      0.54     &       0.18    \\
        GDE &  0.57     &   0.20        &      0.48     &       0.17    \\
        \bottomrule
    \end{tabular}
    \label{table0.1Cont}
\end{table}

We can see that, for the uncontaminated data, SCM has, in median, the least relative error
and the robust versions follow with a small increase. For the data with
the relatively small contaminations of $\ve = 0.05$ and $\ve = 0.1$, we already see a strong
increase in the relative errors for SCM, while there is almost no change for
MCD and GDE.

All the differences in the median values are highly significant w.r.t.\
Wilcoxon signed-rank test, thanks to the high number of random models, with
even the largest p-value being less than $10^{-3}$. Thus, at least in this
particular setting, GDE seems to be superior to MCD.
To better see the dependence of the error on the contamination rate, we plotted the
errors of the 200 models as dodged boxplots grouped by $\ve$ and the estimators.
For $\BB$ see Figure \ref{figBoxB}, for $\See$
see Figure \ref{figBoxCe}. The horizontal red dashed line indicates the
relative error of 1, i.e., everything above this line is only of very limited
use. Note, that the y-axis is logarithmic. Up to the contamination rate of $\ve
= 0.1$, the robustifications are clearly more robust than the baseline SCM, but
starting with $\ve = 0.2$, even the robust implementations return data with RFE
clearly above one.

\section{Conclusion}
Causality is nowadays successfully applied in many branches of science. To
increase accuracy and widen the field of application, models allowing for
feedback loops and hidden confounders need to be investigated more. We have
focused on the LLC algorithm, which satisfies those two requirements. In
particular, we have considered the robustness of LLC.
The robustness of an estimator refers to its capability to remain mainly
unperturbed by contamination of the measurement data with outliers. For the
given scenario, we considered three different types of contamination.  The
robustness metric used here is the estimator's BP. It was shown that the LLC
estimator $\llce$ is not robust, i.e., that its BP is zero. In a sense, it is
even less robust than the covariance estimators it is based on. Finally, we
have performed experiments investigating the effect of substituting robust
covariance estimators into LLC. We focused on the robust estimators MCD and
GDE, where we also provided an implementation for the latter. Those robust
modifications of LLC are showing a clear improvement for lower contamination
rates. In particular, in the considered scenario, the GDE-based LLC estimator
outperforms the one based on MCD. In future work, we will investigate additional
robustness metrics like those bound to the influence function approach.

The implementation of those extensions of LLC also contains a CCP
version of GDE for the multivariate normal case. To improve
reproducibility and encourage further research, the source code is made freely
available on GitHub.

\printbibliography

\end{document}